\documentclass[preprint,3p,times]{elsarticle}
\RequirePackage{multirow,booktabs,subfigure,color,array,hhline,makecell}
\usepackage{amssymb}
\usepackage{amsmath}
\usepackage{graphicx}
\usepackage{amsthm}
\usepackage{mathrsfs}
\usepackage{indentfirst}
\usepackage[colorlinks,citecolor=blue,urlcolor=blue]{hyperref}
\allowdisplaybreaks
\usepackage[table]{xcolor}
\usepackage{tikz-network}
\usepackage{pgf}
\usepackage{tikz}
\usepackage{subfigure}
\usepackage[T1]{fontenc}
\usepackage[utf8]{inputenc}
\usetikzlibrary{arrows, decorations.pathmorphing, backgrounds, positioning, fit, petri, automata}
\usetikzlibrary{shadows,arrows,positioning}
\RequirePackage{algorithm,algpseudocode}
\allowdisplaybreaks
\usepackage{changes}

\usepackage{setspace}
\newtheorem{thm}{Theorem}[section]

\newtheorem{lem}[thm]{Lemma}

\newtheorem{rem}[thm]{Remark}

\allowdisplaybreaks[4]
\DeclareMathOperator{\diag}{diag} 
\AtBeginDocument{%
	\providecommand\BibTeX{{%
			\normalfont B\kern-0.5em{\scshape i\kern-0.25em b}\kern-0.8em\TeX}}}
\journal{~}

\begin{document}

\begin{frontmatter}
\title{A Subsampled Davis–Kahan Bound for Large-Scale Eigenspace Estimation}
\author[label1]{Huan Qing\corref{cor1}}
\ead{qinghuan@cqut.edu.cn; qinghuan@u.nus.edu}
\cortext[cor1]{Corresponding author.}
\address[label1]{School of Economics and Finance, Chongqing University of Technology, Chongqing, 400054, China}

\begin{abstract}
The Davis-Kahan theorem is a fundamental tool in spectral analysis, providing quantitative control over the distance between the eigenspaces of a symmetric matrix and its perturbation. However, when the matrix dimension is large, computing leading eigenvectors is computationally expensive, limiting the practical use of spectral methods in modern large-scale applications. This paper addresses this problem by proposing an independent Bernoulli sampling scheme and proves that the leading left singular vectors of the subsampled matrix faithfully approximate the target subspace of a low-rank symmetric matrix. Our main result is a subsampled Davis--Kahan bound that gives an explicit error bound depending directly on the sampling probability. The bound reveals the trade-off: the computational cost scales linearly with the sampling probability, while the statistical error scales as the inverse square root of the sampling probability. Our result thus extends the Davis-Kahan theorem to the subsampled setting, enabling scalable spectral analysis of large-scale symmetric matrices.
\end{abstract}

\end{frontmatter}
\section{Introduction}\label{sec:intro}
The classical Davis--Kahan theorem \cite{davis1970} provides the foundation for perturbation analysis of symmetric matrices, offering quantitative control over the distance between the eigenspaces of a symmetric matrix and its perturbed version. This theory was later extended by \cite{wedin1972} to general matrices, where the analysis naturally shifts from eigenvectors to the left and right singular subspaces. For statistical applications, \cite{yu2015} developed a widely used variant that improves the classical results in two important ways: it states the separation condition solely in terms of the eigenvalues of the true population matrix, and it keeps the flexibility of using the smaller of the spectral and Frobenius norms for the perturbation. \cite{yu2015}'s work also discusses extensions to non-square matrices. Collectively, these perturbation bounds have become essential tools for establishing theoretical guarantees for spectral analysis in contemporary statistics.

For network community detection, the stochastic block model (SBM) \cite{holland1983stochastic} and its degree-corrected variant (DCSBM) \cite{karrer2011stochastic} serve as benchmark generative models. Under SBM, \cite{rohe2011} used the theorem to study spectral clustering, and \cite{lei2015} proved the consistency of spectral clustering for both SBM and DCSBM. \cite{joseph2016} further studied regularization in spectral clustering within the SBM framework. Under DCSBM, \cite{qin2013} analyzed regularized spectral clustering, and \cite{jin2015} developed the SCORE method using eigenvector perturbation bounds. \cite{binkiewicz2017} added node covariates to spectral clustering and derived error rates using Davis--Kahan; this covariate-assisted approach has been further extended to multi-layer networks \cite{xu2023covariate}. For more complex settings, the theorem has also been used to establish consistency for robust clustering under sub-Gaussian mixture models with outliers \cite{srivastava2023robust}, motif-based spectral clustering under weighted SBM \cite{guo2024efficacy},  distributed spectral clustering for large-scale networks under SBM and DCSBM \cite{wu2023distributed}, spectral clustering in stock co-jump network analysis with a degree-corrected block model under dependent multivariate Poisson edges \cite{ding2024}, spectral clustering or co-clustering for multi-layer undirected or directed networks under multi-layer SBM or its variants \cite{paul2020spectral,lei2023,su2024,qing2025a,qing2025b}, dynamic spectral clustering for time-varying SBM \cite{lin2026dynamic}, regularized spectral clustering for signed networks \cite{cucuringu2021regularized}, spectral embedding for planted pseudo-clique detection in random dot product graphs \cite{qi2025detection}, a unified framework for community detection and model selection across blockmodels \cite{bhadra2026unified}, a unified spectral embedding framework under the generalized random dot product graph \cite{rubindelanchy2022statistical}, and tensor spectral methods for network-valued data in the semi-symmetric tensor PCA framework \cite{weylandt2026multivariate}. A comprehensive overview of spectral methods for data science is given in \cite{chen2021spectral}. For an introduction to spectral clustering, see \cite{vonluxburg2007}.

The Davis--Kahan theorem provides a powerful theoretical tool for analyzing the error of spectral methods. However, these spectral methods themselves require computing the leading eigenvectors of a large \(p\times p\) symmetric matrix. Even when the target rank \(d\) is much smaller than \(p\), a full eigendecomposition still requires \(O(dp^2)\) operations. For large \(p\), this quickly becomes infeasible. For large-scale matrix computations, randomized and subsampling methods have become popular tools for reducing computational cost \cite{mahoney2011, drineas2016}. The main idea is to build a much smaller ``sketch" of the original matrix that still keeps its key features. Two common approaches are random projection and random sampling. Random projection lowers the matrix size by multiplying it with a random matrix \cite{halko2011, witten2015, martinsson2016}. Random sampling picks a subset of the entries, columns, or rows, usually with carefully chosen non-uniform probabilities to keep the error under control \cite{drineas2006, drineas2011, drineas2012}. These ideas have been used successfully in least-squares problems \cite{drineas2006, drineas2011, drineas2012} and low-rank matrix approximation \cite{halko2011, martinsson2016, witten2015}.

In network analysis, several recent papers have used these randomized tools for accelerating spectral clustering, but they take different paths. \cite{zhang2022randomized} studied undirected networks under SBM and DCSBM. Their method either projects the adjacency matrix onto a lower-dimensional space or samples its entries to make the matrix sparser, which speeds up the matrix decomposition while keeping good theoretical performance. \cite{su2025randomized} further extended \cite{zhang2022randomized}'s work to multi-layer networks under the multi-layer SBM. \cite{deng2024} chose a different path: they randomly select a small group of nodes and build a rectangular matrix that records the connections from all nodes to the sampled ones. They then apply a truncated singular value decomposition (SVD) followed by k-means clustering to obtain community labels, and establish theoretical guarantees for this approach under SBM and DCSBM. Although these methods work well in practice, their theoretical analyses and error bounds are model-specific, relying on quantities particular to SBM or its variants. Consequently, their results cannot be extended to spectral methods applied to other types of matrices, nor can they be used to understand column subsampling in more general settings. Furthermore, over the past decade, \cite[Theorem~2]{yu2015} has been a standard tool in the theoretical analysis of many of the aforementioned spectral methods that estimate the leading eigenvectors of a symmetric matrix. However, those methods themselves become computationally expensive for large-scale matrices. While column subsampling offers a natural acceleration strategy, it breaks the symmetry of the matrix: the resulting object of analysis is rectangular, not symmetric. The setting of \cite[Theorem~2]{yu2015}---which compares two symmetric matrices---therefore does not directly apply. A new theoretical framework is needed to handle this cross-type comparison. These motivate our work, which develops such a framework for using the leading left singular vectors of a column-subsampled rectangular matrix to estimate the leading eigenspace of the original low-rank symmetric matrix.

Specifically, we extend the classical Davis--Kahan theorem to the large-scale setting for symmetric matrices, where the leading eigenspace is estimated from a randomly selected subset of columns rather than from the full matrix. Let \(\hat{\Sigma} \in \mathbb{R}^{p \times p}\) be the observed symmetric data matrix, and let \(\Sigma \in \mathbb{R}^{p \times p}\) denote its noiseless low-rank version, with \(\operatorname{rank}(\Sigma)=d \ll p\), whose leading \(d\)-dimensional eigenspace is our primary object of interest. To reduce computational cost, we apply an independent Bernoulli sampling scheme to the columns of \(\hat{\Sigma}\): each column is retained with probability \(\alpha \in (0,1]\), independently of all others. Retaining \(n\) columns with \(n \leq p\), we form a rectangular matrix \(H \in \mathbb{R}^{p \times n}\), and then use the leading \(d\) left singular vectors of \(H\) as proxies for the target eigenspace of \(\Sigma\). This extension, however, is not a direct application of the classical result. The key issue is that column subsampling breaks the symmetry of the data matrix: the matrix we operate on is no longer the square symmetric matrix \(\hat{\Sigma}\) or \(\Sigma\), but the rectangular matrix \(H\), whose singular vectors are governed by a perturbation theory different from that for symmetric eigenproblems. As a result, the classical Davis--Kahan theorem, which is designed specifically for symmetric matrices, no longer applies directly. To address this, our main result establishes a subsampled Davis--Kahan bound that controls the Frobenius-norm error between the target \(d\)-dimensional eigenspace of the noiseless low-rank symmetric matrix \(\Sigma\) and the leading \(d\) left singular vectors of the subsampled rectangular matrix \(H\). This theorem holds for any symmetric matrix without imposing any generative model assumptions. The resulting error bound depends explicitly on the sampling probability, the noise level, and the spectral gap, and reveals a clear quantitative trade-off: the statistical error grows inversely with the square root of the sampling probability, while the computational cost decreases proportionally to the sampling probability. We illustrate the practical utility of this general framework through an application to the stochastic block model and validate our theoretical findings with numerical experiments.

The rest of this paper is organized as follows. Section~\ref{sec:setup} introduces the problem formulation. Section~\ref{sec:main} presents the main theoretical results. Section~\ref{sec:discussion} discusses the practical implications of our bound and a detailed application to the stochastic block model. Section~\ref{sec:numexp} provides numerical experiments. Section~\ref{sec:conclusion} concludes the paper. All technical proofs are collected in the Appendix.
\section{Problem Formulation}
\label{sec:setup}

Let $\Sigma \in \mathbb{R}^{p \times p}$ be a symmetric matrix of rank $d \ll p$. Write its compact spectral decomposition as
\begin{align*}
\Sigma = V \Lambda V^\top,
\end{align*}
where $V \in \mathbb{R}^{p \times d}$ has orthonormal columns and $\Lambda = \diag(\lambda_1,\lambda_2,\dots,\lambda_d)$ contains the $d$ nonzero eigenvalues. The column space of $V$ is often the quantity of primary interest: in community detection, for instance, it encodes the community membership structure of the nodes. Let
\begin{align*}
\delta := \min_{1 \le i \le d} |\lambda_i| > 0,
\end{align*}
where we do not require $\Sigma$ to be positive semidefinite.

We observe a noisy version of $\Sigma$:
\begin{align*}
\hat{\Sigma} = \Sigma + E, 
\end{align*}
with $E \in \mathbb{R}^{p \times p}$ is symmetric. Let $\|E\|_2$ and $\|E\|_F$ denote $E$'s spectral and Frobenius norms. To estimate the target eigenspace $V$, a standard approach is to compute the leading $d$ eigenvectors $\hat V$ of the $p \times p$ symmetric matrix $\hat{\Sigma}$. The classical Davis--Kahan theorem bounds the error between $V$ and $\hat V$ in terms of $\|E\|_2$ and $\delta$; see \cite{yu2015} for a refined variant more directly suited to statistical applications. However, computing $\hat V$ from the $p\times p$ symmetric matrix $\hat{\Sigma}$ costs $O(dp^2)$ time, which is prohibitive when $p$ is large.

To reduce the cost, we randomly subsample the columns of $\hat{\Sigma}$.
Let $\zeta_1,\zeta_2,\dots,\zeta_p \in \{0,1\}$ be independent Bernoulli
variables with $\mathbb{P}(\zeta_i=1)=\alpha$ for $1\leq i\leq p$, where $\alpha\in (0,1]$ is the
sampling probability. Retain column $i$ if $\zeta_i=1$, and let $S \in \mathbb{R}^{p \times n}$ be the column
selection matrix whose columns are the selected standard basis vectors
$\{e_i : \zeta_i=1\}$; hence $S$ satisfies $S^\top S = I_n$ and $SS^{\top}=\mathrm{diag}(\zeta_1,\zeta_2,\dots,\zeta_p)$. Here, 
\begin{align*}
n =\sum_{i=1}^p \zeta_i
\end{align*}
is the number of retained columns. The subsampled matrix is then
defined as 
\begin{align*}
H = \hat{\Sigma} S \in \mathbb{R}^{p \times n}
\end{align*}
We then use $H$'s leading $d$ left singular vectors $\hat V_H \in \mathbb{R}^{p \times d}$ as an estimator of $V$.

Constructing $H$ by extracting the selected columns from $\hat{\Sigma}$
costs $O(p n)$ operations. Computing the leading $d$ left singular
vectors of $H$ via a standard truncated singular value decomposition
costs $O(d p n)$ operations. Since $\mathbb{E}[n] = \alpha p$, the expected total time complexity is $O(\alpha d p^2)$, which can be substantially less than the full-matrix approach requiring $O(d p^2)$ operations when $\alpha$ is small.

The central question is whether the leading $d$ left singular vectors
$\hat V_H$ of the $p\times n$ subsampled matrix $H$ faithfully approximate the leading $d$ eigenvectors $V$ of the $p\times p$ symmetric matrix $\Sigma$ despite column subsampling, and how the approximation error depends on the sampling probability $\alpha$. We establish a non-asymptotic bound on the Frobenius-norm distance between $V$ and $\hat V_H$ (up to an orthogonal rotation) that makes this dependence explicit.

Table~\ref{tab:notation} summarizes the notations used throughout.
\begin{table}[htbp]
\centering
\caption{Table of notations.}
\label{tab:notation}
\resizebox{1\textwidth}{!}{%
\begin{tabular}{l >{\raggedright\arraybackslash}p{0.48\textwidth} l >{\raggedright\arraybackslash}p{0.48\textwidth}}
\toprule
Symbol & Definition & Symbol & Definition \\
\midrule
$\mathbb{R}$ & Real numbers &
$W$ & Noise part, $W = E S \in \mathbb{R}^{p \times n}$ \\
$p$ & Matrix dimension &
$\hat{V}_H$ & Top-$d$ left singular vectors of $H$ \\
$d$ & Target rank, $d \ll p$ &
$V_B$ & Top-$d$ left singular vectors of $B$ \\
$\rho$ & Sparsity parameter, $0 < \rho \le 1$ &
$\mu$ & Incoherence parameter, $\mu = \frac{p}{d} \|V\|_{2,\infty}^2$ \\
$\Sigma$ & True low-rank symmetric matrix, $\Sigma \in \mathbb{R}^{p \times p}$, $\operatorname{rank}(\Sigma)=d$ &
$\hat{O}$ & Orthogonal rotation matrix, $\hat{O} \in \mathbb{R}^{d \times d}$, $\hat{O}^\top \hat{O} = I_d$ \\
$E$ & Symmetric noise matrix&
$\sigma_i(\cdot)$ & $i$-th largest singular value \\
$\hat{\Sigma}$ & Observed matrix, $\hat{\Sigma} = \Sigma + E$ &
$\|\cdot\|_2$ & Spectral norm \\
$V$ & Top-$d$ eigenvector matrix, $V \in \mathbb{R}^{p \times d}$, $V^\top V = I_d$ &
$\|\cdot\|_F$ & Frobenius norm \\
$\Lambda$ & Eigenvalue diagonal matrix, $\Lambda = \diag(\lambda_1,\dots,\lambda_d)$ &
$\|\cdot\|_{2,\infty}$ & Entrywise max row norm, $\|M\|_{2,\infty} = \max_i \|M_{i,:}\|_F$ \\
$\delta$ & Minimum absolute eigenvalue, $\delta = \min_{1 \le i \le d} |\lambda_i|$ &
$\mathbb{P}$ & Probability \\
$\alpha$ & Bernoulli sampling probability, $0 < \alpha \le 1$ &
$\mathbb{E}$ & Expectation \\
$\{\zeta_i\}_{i=1}^p$ & Bernoulli variables, $\zeta_i \sim \mathrm{Bernoulli}(\alpha)$ &
$C$ & Generic positive constant \\
$n$ & Sampled columns, $n = \sum_{i=1}^p \zeta_i$ &
$Z$ & Membership matrix, $Z \in \{0,1\}^{p \times d}$ \\
$S$ & Column selection matrix, $S \in \mathbb{R}^{p \times n}$, $S^\top S = I_n$ &
$B_0$ & Connectivity matrix, $B_0 \in [0,1]^{d \times d}$ \\
$H$ & Subsampled matrix, $H = \hat{\Sigma} S \in \mathbb{R}^{p \times n}$ &
$g_i$ & True community label of node $i$ \\
$B$ & Signal part, $B = \Sigma S \in \mathbb{R}^{p \times n}$ &
$\hat{g}_i$ & Estimated community label of node $i$ \\
\bottomrule
\end{tabular}%
}
\end{table}
\section{Main Results}
\label{sec:main}

This section presents the main theoretical results of this paper. We begin by establishing three key lemmas that serve as the building blocks of our analysis. These lemmas address, respectively, the well-definedness of the subsampled estimator, the preservation of the target subspace under Bernoulli sampling, and the deterministic perturbation bound that controls the effect of noise. We then combine these ingredients to prove the main theorem, which gives an explicit error bound for the subsampled eigenspace estimator.

\subsection{Preliminary Lemmas}
We start with a lemma that ensures the subsampled matrix has at least \(d\) columns with high probability. This is necessary for the top-\(d\) singular vectors to be well-defined.
\begin{lem}\label{lem:n_lower_bound}
For  $p \ge 2$ and $d \ge 1$, suppose that
\begin{align}\label{eq:alpha_condition_n}
\alpha\ge 16 \frac{d \log p}{p}.
\end{align}
Then, we have
\begin{align*}
\mathbb{P}(n \ge d) \ge 1 - p^{-2d}.
\end{align*}
\end{lem}

The following lemma is the key technical ingredient of our analysis.
It establishes that, with high probability, the randomly sampled
columns preserve the \(d\)-dimensional structure of the target
subspace. 
\begin{lem}\label{lem:subspace}
Let $V \in \mathbb{R}^{p \times d}$ have orthonormal columns and define $\mu := \frac{p}{d}\|V\|_{2,\infty}^2$. Let $S \in \mathbb{R}^{p \times n}$ be the random column selection matrix obtained by independent Bernoulli sampling with probability $\alpha\in (0,1]$. If $\alpha=1$, then $\sigma_d(V^\top S)=1$ deterministically. If $0<\alpha<1$ and
\begin{align}
\alpha\ge 16\,\frac{\mu d \log p}{p}, \label{eq:alpha_subspace}
\end{align}
then, with probability at least $1 - d p^{-2}$, we have
\begin{align*}
\sigma_d(V^\top S) \ge \sqrt{\frac{\alpha}{2}}.
\end{align*}
\end{lem}

We shall also need a deterministic perturbation bound for rectangular
matrices.
\begin{lem}\label{lem:wedin}
Let $B \in \mathbb{R}^{p \times n}$ have rank $d$, and let $H = B + W$, with $W \in \mathbb{R}^{p \times n}$. Denote by $V_B \in \mathbb{R}^{p \times d}$ and $\hat{V}_H \in \mathbb{R}^{p \times d}$ the matrices of the leading $d$ left singular vectors of $B$ and $H$, respectively. If
\begin{align*}
\|W\|_2 \le \frac{1}{2}\,\sigma_d(B),
\end{align*}
then there exists an orthogonal matrix $\hat{O} \in \mathbb{R}^{d \times d}$ such that
\begin{align*}
\|\hat{V}_H \hat{O} - V_B\|_F
\le 
\frac{2\sqrt{2}\,\min\{\sqrt{d}\,\|W\|_2,\; \|W\|_F\}}
{\sigma_d(B)}.
\end{align*}
\end{lem}
\subsection{Main Theorem}

We are now ready to state the main theorem of the paper. This theorem combines the three lemmas above to give a complete characterization of the estimation error of the subsampled eigenspace.
\begin{thm}\label{thm:main}
Let \(\alpha \in (0,1]\).
\begin{itemize}
    \item If \(\alpha = 1\), then deterministically there exists an orthogonal matrix \(\hat O \in \mathbb{R}^{d \times d}\) such that
    \[
    \|\hat V_H \hat O - V\|_F 
    \le \frac{2\sqrt{2}\,\min\{\sqrt{d}\,\|E\|_2,\|E\|_F\}}{\delta}.
    \]
    \item If \(0<\alpha<1\), suppose that
    \[
    \alpha \ge \max\left\{16\frac{\mu d\log p}{p},\; 8\left(\frac{\|E\|_2}{\delta}\right)^2\right\}.
    \]
    Then, with probability at least \(1 - p^{-2d} - d p^{-2}\), there exists an orthogonal matrix \(\hat O \in \mathbb{R}^{d \times d}\) such that
    \[
    \|\hat V_H \hat O - V\|_F 
    \le \frac{4\,\min\{\sqrt{d}\,\|E\|_2,\|E\|_F\}}{\delta \sqrt{\alpha}}.
    \]
\end{itemize}
\end{thm}

The theorem quantifies the subsampling trade-off in explicit terms. When \(\alpha=1\), the result recovers the classical Davis-Kahan bound of \cite{yu2015} without any additional assumptions on the noise level. For the subsampling regime \(0<\alpha<1\), the error is inflated by \(1/\sqrt{\alpha}\) while the expected computational cost is reduced by \(\alpha\). The condition \(16\mu d\log p/p\) ensures the sampled columns faithfully capture the target subspace, while \(8(\|E\|_2/\delta)^2\) prevents the noise from overwhelming the signal in the reduced sample. These two requirements together determine the minimal sampling probability for reliable recovery. The message is straightforward: randomization saves computation at a quantified statistical cost, and using more columns improves accuracy at a quantified rate.
\begin{rem}(On the feasibility of rank.)
Combining the condition \(\alpha\ge \max\left\{16\frac{\mu d\log p}{p},\; 8\left(\frac{\|E\|_2}{\delta}\right)^2\right\}\) in Theorem~\ref{thm:main} with \(\alpha<1\) immediately implies
\begin{align}
\mu d <\frac{p}{16\log p}. \label{eq:rank_mu}
\end{align}
This condition is very mild. When \(\mu = O(1)\), Equation \eqref{eq:rank_mu} allows \(d\) to grow almost linearly with \(p\), up to a logarithmic factor. Most practical low-rank settings---fixed \(d\), slowly growing \(d\), and many sublinear regimes---easily satisfy it. The condition only fails in the regime where \(d\) is nearly as large as \(p\), which is rarely of interest in high-dimensional statistics. If \(\mu\) grows with \(p\), the effective bound becomes \(O(p/(\mu\log p))\), which is tighter and remains sufficient in many practical settings.
\end{rem}

\begin{rem}(On the choice of the constant in the sampling condition.)
The specific value of the constant in the conditions of Lemmas~\ref{lem:n_lower_bound} and~\ref{lem:subspace} is not unique. To illustrate this, suppose we replace the condition in Lemma~\ref{lem:n_lower_bound} by the weaker requirement \(\alpha\ge 8d\log p/p\). Then the Chernoff bound yields the weaker probability guarantee \(\mathbb{P}(n\ge d)\ge 1-p^{-d}\) instead of \(1-p^{-2d}\). Likewise, if in Lemma~\ref{lem:subspace} we weaken the condition to \(\alpha\ge 8\mu d\log p/p\), the matrix Chernoff bound gives \(\mathbb{P}(\sigma_d(V^\top S)\ge \sqrt{\alpha/2})\ge 1-d/p\) in place of \(1-dp^{-2}\).

Consequently, under the milder sampling condition
\[
\alpha \ge \max\left\{8\frac{\mu d\log p}{p},\; 8\left(\frac{\|E\|_2}{\delta}\right)^2\right\},
\]
Theorem~\ref{thm:main} remains valid, with the probability bound in the subsampling regime becoming \(1-p^{-d}-dp^{-1}\).

This reveals a transparent trade-off: using the smaller constant 8 (instead of 16) allows a lower sampling rate \(\alpha\), but at the cost of a weaker probability guarantee. In this paper we choose the constant 16 to obtain the stronger probability bound \(1-p^{-2d}-dp^{-2}\), which decays faster with \(p\).  In general, the condition on the Bernoulli sampling probability \(\alpha\) in Theorem~\ref{thm:main} can be relaxed whenever one is willing to accept a slightly weaker probabilistic guarantee.
\end{rem}

\begin{rem}[On the cross-type comparison in Theorem~\ref{thm:main}]
Unlike \cite[Theorem~2]{yu2015}, which compares the eigenvectors of two symmetric matrices of the same dimension, Theorem~\ref{thm:main} involves the original symmetric matrix $\Sigma\in\mathbb{R}^{p\times p}$ and the rectangular subsampled matrix $H\in\mathbb{R}^{p\times n}$, and provides a direct bound between the eigenvectors of $\Sigma$ and the left singular vectors of $H$. This is made possible because, under the sampling condition and on a high-probability event, the column space of the subsampled signal matrix $B=\Sigma S$ coincides with the target eigenspace $\operatorname{span}(V)$; together with the noise bound $\|ES\|_2 \le \delta\sqrt{\alpha/8}$, this ensures that Wedin's perturbation theory \cite{wedin1972} applies and allows the error to be transferred from the subsampled matrix back to the original symmetric matrix.
\end{rem}
\section{Practical Implications}\label{sec:discussion}
Our theorem provides a unified, subsampled variant of the classical
Davis--Kahan theorem for symmetric matrices. As such, it can be
directly employed in any statistical procedure whose theoretical
analysis relies on controlling the perturbation of eigenspaces in the
Frobenius norm (with perturbations bounded in either the spectral or
Frobenius norm). To illustrate the broad applicability of this result,
it is helpful to categorize spectral problems into two classes according to their analysis structure, both of which are systematically studied in the monograph \cite{chen2021spectral}.

The first class consists of problems where the primary objective is
to estimate the underlying principal subspace. As analyzed in
Chapter~3 of \cite{chen2021spectral}, typical examples include
\emph{principal component analysis} under the spiked covariance
model, \emph{low-rank matrix denoising} with additive noise, and
\emph{tensor completion} after proper matricization and diagonal
deletion (which constructs a symmetric matrix whose leading
eigenspace encodes the desired tensor factors). The core analysis
in each case reduces to bounding the noise \(\|E\|\) and invoking
the Davis--Kahan theorem. Our Theorem~\ref{thm:main} therefore
applies directly, yielding a column-subsampled version of the
original guarantee.

The second class encompasses problems that follow a common
\emph{two-stage} structure: they first compute the leading
eigenvectors of a data-derived symmetric matrix to obtain a
low-dimensional subspace estimate, and then apply a task-specific
procedure---such as \(k\)-means clustering, entry-wise ratio-taking,
or normalization---to extract the quantities of interest. The
first stage is typically the computational bottleneck. As analyzed
in Chapter~3 of \cite{chen2021spectral}, this class includes
spectral clustering algorithms for Gaussian mixture models (using
the Gram matrix of the data), as well as many spectral methods
for community detection in network analysis. Representative examples include spectral clustering in SBM
~\cite{rohe2011,lei2015,joseph2016}, spectral clustering in 
DCSBM~\cite{qin2013,jin2015}, motif-based spectral clustering
~\cite{guo2024efficacy}, and spectral clustering or co-clustering for
multi-layer networks~\cite{lei2023,su2024}. A broader range of
such examples is surveyed in the Introduction. In all these cases,
the first-stage analysis---bounding \(\|E\|\), determining the
spectral gap \(\delta\), and invoking the Davis--Kahan theorem---
reduces to estimating the leading eigenspace (or its projection
matrix) of a symmetric matrix. Our Theorem~\ref{thm:main} replaces
the full eigenvector matrix \(\hat V\) with the subsampled estimator
\(\hat V_H\), reducing the complexity from \(O(dp^2)\) to
\(O(\alpha d p^2)\) while keeping the subspace error under control;
the downstream task then proceeds unchanged. The acceleration
factor is therefore \(1/\alpha\): a smaller sampling probability
yields faster computation at the cost of a larger error bound.

To illustrate concretely how the second stage can be analyzed and how the subspace error propagates to the final performance, we next work through a detailed application to the stochastic block model.
\subsection{Application to Stochastic Block Models}\label{sub:SBM}

We now demonstrate how our general bound applies to the stochastic block model (SBM), a standard setting for community detection. Let \(\widehat{\Sigma}\in\{0,1\}^{p\times p}\) be the observed symmetric adjacency matrix, and let its expectation be
\begin{align*}
\Sigma = \rho Z B_0 Z^\top,
\end{align*}
where:
\begin{itemize}
\item \(d\) is the number of communities, assumed to be \(O(1)\);
\item \(Z\in\{0,1\}^{p\times d}\) is the membership matrix, with \(Z_{ik}=1\) if node \(i\) belongs to community \(k\), and \(Z_{ik}=0\) otherwise;
\item \(B_0\in\mathbb{R}^{d\times d}\) is symmetric with entries in \([0,1]\);
\item \(\rho\in(0,1]\) controls the overall sparsity.
\end{itemize}

Let \(g_i\in\{1,2,\dots,d\}\) denote the community label of node \(i\), and let \(p_k = \sum_{i=1}^p Z_{ik}\) be the size of community \(k\). We assume balanced communities, i.e., \(p_k = O(p/d)\) for all \(k\), which with \(d=O(1)\) simply means each community has size \(O(p)\).

The subsampled estimator \(\hat V_H\) is obtained from the leading left singular vectors of \(H=\widehat{\Sigma}S\). To recover the community labels, we run \(k\)-means on the rows of \(\hat V_H\) with \(d\) clusters. The full procedure is described in Algorithm~\ref{alg:sbm}.

\begin{algorithm}
\caption{Column-Subsampled Spectral Clustering (CSSC) for SBM}
\label{alg:sbm}
\begin{algorithmic}[1]
\State \textbf{Input:} Adjacency matrix \(\widehat{\Sigma}\in\{0,1\}^{p\times p}\), number of communities \(d\), sampling probability \(\alpha\)
\State Draw \(\zeta_1,\zeta_2,\dots,\zeta_p\sim \mathrm{Bernoulli}(\alpha)\) independently and form the column selection matrix \(S\)
\State Compute \(H = \widehat{\Sigma}S\)
\State Compute the leading \(d\) left singular vectors \(\hat V_H\in\mathbb{R}^{p\times d}\) of \(H\)
\State Run \(k\)-means with \(d\) clusters on the rows of \(\hat V_H\), obtaining cluster labels \(\hat{g}\in\{1,2,\dots,d\}^p\)
\State \textbf{Return:} \(\hat{g}\)
\end{algorithmic}
\end{algorithm}

This is the subsampling counterpart of the standard spectral clustering algorithm analyzed in \cite{lei2015}, where the only difference is that we operate on the subsampled \(p\times n\) matrix \(H\) rather than the full \(p\times p\) symmetric matrix \(\widehat{\Sigma}\).

The following theorem establishes the misclustering rate of Algorithm~\ref{alg:sbm}.

\begin{thm}\label{thm:sbm}
Consider an SBM with \(d=O(1)\), balanced communities, and the smallest absolute eigenvalue of \(B_0\) bounded below by a constant. If
\begin{align*}
\alpha\ge C\frac{\log p}{\rho p}
\end{align*}
for a sufficiently large constant \(C\), then with high probability, we have
\begin{align*}
\frac{1}{p}\min_{\pi}\sum_{i=1}^p \mathbf{1}\{\hat{g}_i \neq \pi(g_i)\}
= O\!\left(\frac{\log p}{\rho\, p\, \alpha}\right),
\end{align*}
where \(\pi\) ranges over permutations of \(\{1,2,\dots,d\}\).
\end{thm}

In the dense network regime where \(\rho\) is constant, the rate simplifies to \(O\!\left(\frac{\log p}{p\,\alpha}\right)\). To drive this error to zero as \(p\to\infty\), it suffices that \(\alpha \gg \log p / p\). For instance, taking
\begin{align*}
\alpha = \frac{\log^{1+\epsilon} p}{p} \qquad (\epsilon>0)
\end{align*}
yields a vanishing misclustering rate of \(O(\log^{-\epsilon} p)\). In terms of computational cost, Algorithm~\ref{alg:sbm} requires \(O(\alpha p^2)=O(p\log^{1+\epsilon}p)\) expected time, compared with \(O(p^2)\) for the full spectral method—a speed-up by a factor of \(p/\log^{1+\epsilon}p\). 

\section{Numerical Experiments}\label{sec:numexp}
This section empirically validates the computational-accuracy trade-off established in Theorems~\ref{thm:main} and~\ref{thm:sbm} via comparisons between the column-subsampled spectral clustering (CSSC, Algorithm~\ref{alg:sbm}) and the standard spectral clustering (SC, Algorithm~1 of \cite{lei2015}) under the stochastic block model described in Section~\ref{sub:SBM}. All experiments are performed in MATLAB R2024b on a personal computer (ThinkPad X1). To ensure a fair comparison, both algorithms use the same power iteration routine for subspace extraction, with identical termination criteria. Running time is measured exclusively for the subspace estimation phase: for SC, we record the time to compute the top-$d$ eigenvectors $\hat V$ of the full matrix $\hat\Sigma$; for CSSC, we record the time from constructing the subsampled matrix $H = \hat\Sigma S$ to computing its top-$d$ left singular vectors $\hat V_H$. The subsequent $k$-means post-processing is excluded from the timing and configured identically for both algorithms.

We consider $d=3$ balanced communities, where each node is assigned to one of the three communities with equal probability. The connectivity matrix is set as
\[
B_0 = \begin{pmatrix}
0.6 & 0.2 & 0.2\\
0.2 & 0.6 & 0.2\\
0.2 & 0.2 & 0.6
\end{pmatrix}.
\]

The misclustering error is computed as
\[
\mathrm{Err} = \frac{1}{p} \min_{\pi \in \mathcal{S}_3} \sum_{i=1}^p \mathbf{1}\{\hat{g}_i \neq \pi(g_i)\},
\]
where $g_i$ and $\hat{g}_i$ denote the true and estimated community labels, respectively.

Two numerical experiments are conducted to empirically validate the theoretical trade-off between computational efficiency and estimation accuracy. For each parameter configuration, we report the average misclustering rate and the average running time (based on 100 independent repetitions) for both CSSC and SC.

\emph{Experiment 1: Increasing network size.} To examine the scalability of the algorithms, we fix the sparsity parameter at $\rho = 0.5$, which corresponds to a relatively dense regime where the signal strength is strong and the computational advantage of column subsampling is expected to be pronounced. The sampling probability is set to $\alpha(p) = 20(\log p)^{1.2}/p$, which satisfies the sufficient condition of Theorem~\ref{thm:sbm} with a large margin, since
\[
\frac{\alpha}{\log p / (\rho p)} = 20\rho (\log p)^{0.2} \ge 10(\log 5000)^{0.2}> 6 \qquad \text{for all } p \ge 5000.
\]
We vary $p$ from $5000$ to $50000$ in increments of $5000$, generating $100$ independent SBM adjacency matrices $\widehat{\Sigma}\in\{0,1\}^{p\times p}$ for each value of $p$.

\begin{figure}[!htbp]
\centering
\resizebox{\columnwidth}{!}{
{\includegraphics[width=2\textwidth]{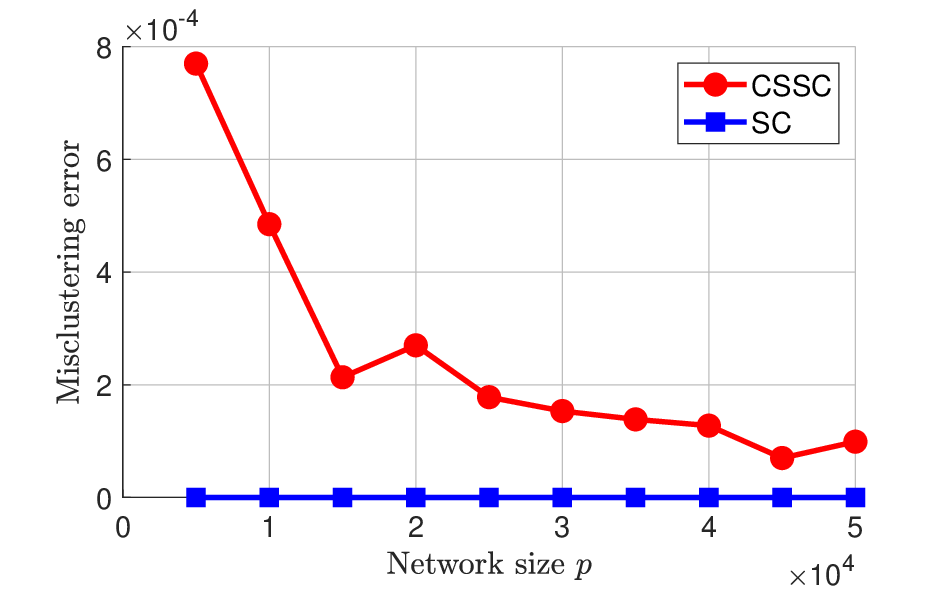}}
{\includegraphics[width=2\textwidth]{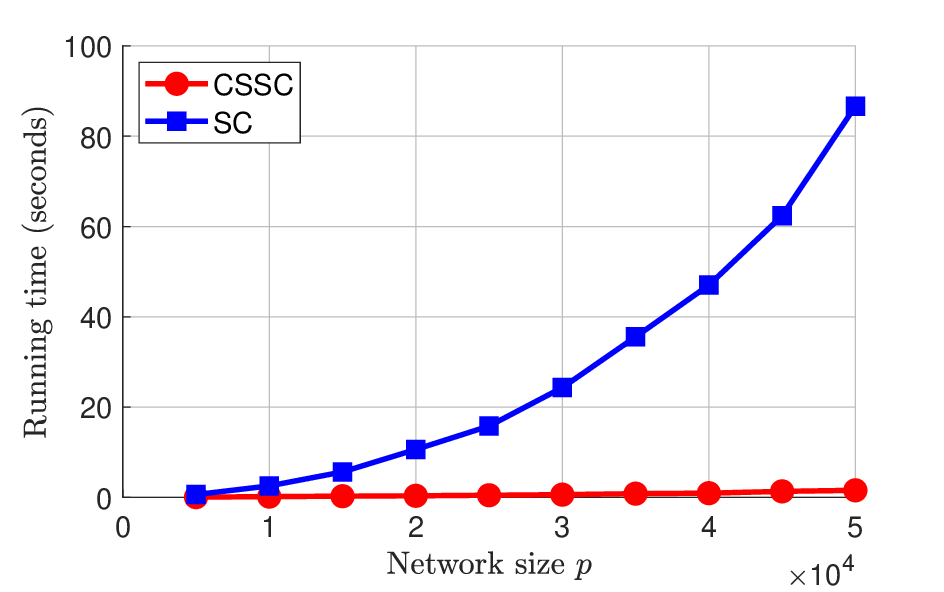}}
}
\caption{Numerical results of Experiment 1.}
\label{fig:ex1} 
\end{figure}

Figure~\ref{fig:ex1} reports the numerical results of Experiment 1. The left panel shows that SC achieves exactly zero misclustering error for all network sizes, while CSSC attains an error that vanishes as $p$ increases and remains very small throughout. This indicates that subsampling maintains a high level of accuracy in the dense regime. The right panel displays the running time. The time for SC increases quickly as $p$ grows, whereas the time for CSSC rises only gradually. Notably, even at the largest network size $p=50000$, the subspace estimation time for CSSC remains well below 2 seconds, demonstrating its remarkable efficiency. As a result, the gap between the two curves becomes wider for larger $p$, meaning that the speed-up factor of CSSC over SC grows with the network size. These observations are consistent with the theoretical trade-off in Theorem~\ref{thm:main}: subsampling yields substantial computational gains with only a negligible, asymptotically vanishing loss in estimation precision.

\emph{Experiment 2: Increasing sparsity.} To isolate the effect of signal strength on the performance of CSSC, we fix the network size at $p = 50000$ and let the sparsity parameter $\rho$ increase from $0.1$ to $1.0$ in steps of $0.1$, while keeping the sampling probability fixed at $\alpha = 20(\log 50000)^{1.2}/50000$. For $\rho \ge 0.2$, the theoretical condition is comfortably met because
\[
\frac{\alpha}{\log p / (\rho p)} = 20\rho (\log 50000)^{0.2} \ge 20 \times 0.2 \times (\log 50000)^{0.2} > 6.
\]
The case $\rho = 0.1$ lies slightly below the theoretical threshold and is included to investigate the empirical behavior of the method in the sparse regime.

\begin{figure}[!htbp]
\centering
\resizebox{\columnwidth}{!}{
{\includegraphics[width=2\textwidth]{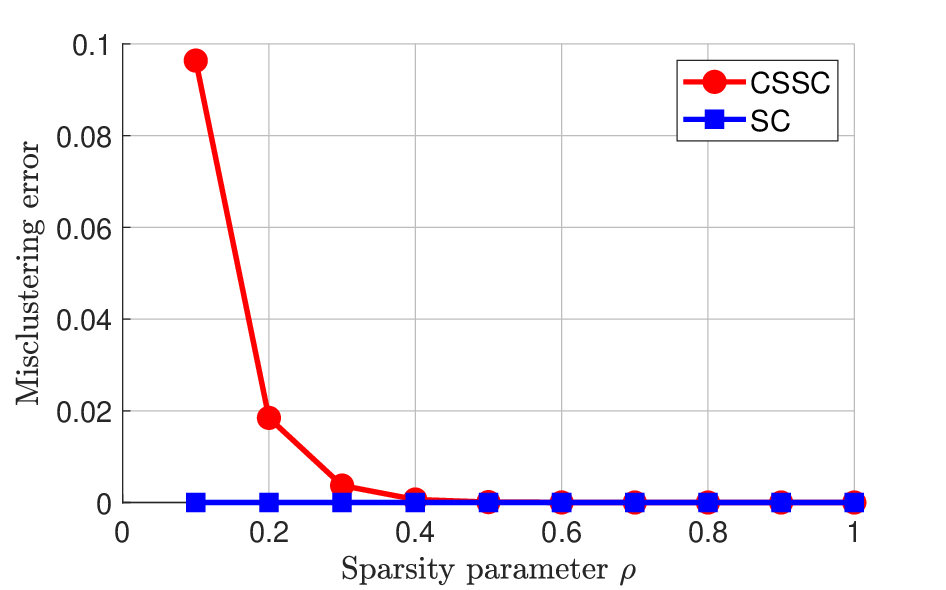}}
{\includegraphics[width=2\textwidth]{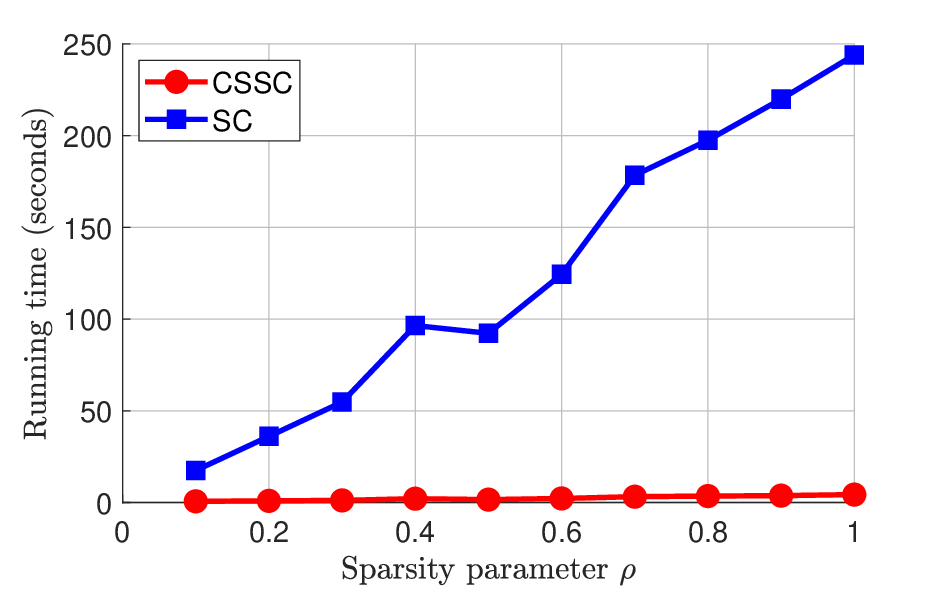}}
}
\caption{Numerical results of Experiment 2.}
\label{fig:ex2} 
\end{figure}

Figure~\ref{fig:ex2} presents the numerical results of Experiment 2. The left panel shows the misclustering error. SC attains exactly zero error throughout. CSSC produces a noticeable error only when $\rho$ is very small; as $\rho$ increases, the error drops quickly and becomes negligible for $\rho \ge 0.2$, reaching zero around $\rho = 0.5$. This pattern reflects the natural signal-strength effect: a larger $\rho$ yields a stronger signal, making the subspace easier to recover under column subsampling. The right panel displays the running time. Both methods become slower as $\rho$ grows, since a denser adjacency matrix contains more nonzeros, which increases the cost of matrix operations. However, the increase for SC is dramatic, with the time approaching 250 seconds when $\rho=1$ (the densest case considered), while CSSC remains below 5 seconds for all $\rho$ values. The difference between the two curves becomes more pronounced as $\rho$ increases, implying that the speed-up factor of CSSC over SC is actually larger in denser networks. This is because the computational burden of SC grows directly with the number of edges, whereas CSSC operates on a much smaller matrix and is therefore far less affected by edge density.
\section{Conclusion}\label{sec:conclusion}
This paper establishes a subsampled Davis--Kahan theorem for
estimating the leading eigenspace of a large symmetric matrix
from a randomly selected subset of its columns. Unlike the setting of \cite[Theorem~2]{yu2015}, which compares the eigenvectors of two symmetric matrices, our result provides a direct bound between the eigenvectors of the original symmetric matrix and the left singular vectors of the rectangular subsampled matrix. The proposed theorem
provides a quantitative characterization of the trade-off between
computational cost and statistical accuracy in large-scale
spectral analysis. Subsampling with probability \(\alpha\) inflates
the Frobenius-norm subspace estimation error by \(1/\sqrt{\alpha}\)
and reduces the expected computational cost by \(\alpha\). This
relationship holds uniformly for any symmetric low-rank signal,
subject to an explicit condition on the sampling probability
stated in the theorem. The framework is broadly applicable to spectral methods that involve the extraction of a leading eigenspace from a symmetric data matrix. We illustrate the practical utility of this framework through a detailed analysis of the stochastic block model, where the subsampled spectral clustering algorithm achieves a vanishing misclustering rate at a substantially reduced computational cost compared with the full spectral method. Numerical experiments support these theoretical findings.


\appendix

\section{Technical Proofs}\label{sec:proofs}
\subsection{Proof of Lemma~\ref{lem:n_lower_bound}}
\begin{proof}
Since $n \sim \mathrm{Binomial}(p, \alpha)$, the standard Chernoff lower-tail bound gives for any $\tau \in (0,1)$,
\begin{align*}
\mathbb{P}\bigl(n \le (1-\tau) p \alpha\bigr) \le \exp\!\left(-\frac{\tau^2}{2} p \alpha\right).
\end{align*}

From Equation \eqref{eq:alpha_condition_n}, we have $p \alpha\ge 16 d \log p$. For $p \ge 2$, $16\log p \ge 16\log 2 > 2$, hence $p\alpha > 2d$. Define
\begin{align*}
\tau := 1 - \frac{d}{p \alpha}.
\end{align*}

Since $d \ge 1$ and $p \alpha > 2d$, we have $0 < d/(p \alpha) < 1/2$, so $\tau \in (1/2, 1) \subset (0,1)$. Thus the Chernoff bound applies. Observe that $\{n < d\} \subseteq \{n \le (1-\tau) p \alpha\}$, because $(1-\tau) p \alpha = d$. Therefore, we have
\begin{align*}
\mathbb{P}(n < d)
\le \mathbb{P}\!\left(n \le (1-\tau) p \alpha\right)\le \exp\!\left(-\frac{\tau^2}{2} p \alpha\right)\le \exp\!\left(-\frac{p \alpha}{8}\right),
\end{align*}
where the last inequality uses $\tau > 1/2$.

Using Equation \eqref{eq:alpha_condition_n} again, we get
\begin{align*}
\exp\!\left(-\frac{p \alpha}{8}\right)
&\le \exp\!\left(-2 d \log p\right)
= p^{-2d}.
\end{align*}

Thus, we get
\begin{align*}
\mathbb{P}(n \ge d) \ge 1 - p^{-2d},
\end{align*}
which completes the proof.
\end{proof}
\subsection{Proof of Lemma~\ref{lem:subspace}}
\begin{proof}
If $\alpha=1$, then $S=I_p$ deterministically, so $\sigma_d(V^\top S)=1$. Now assume $0<\alpha<1$. Let $e_i \in \mathbb{R}^p$ denote the $i$-th standard basis vector and set $u_i = V^\top e_i \in \mathbb{R}^d$ for $i=1,2,\dots,p$. Since $V^\top V = I_d$, we have
\begin{align*}
\sum_{i=1}^p u_i u_i^\top = I_d.
\end{align*}

Define the $d\times d$ random matrix
\begin{align*}
Y := V^\top S S^\top V = \sum_{i=1}^p \zeta_i u_i u_i^\top, 
\end{align*}
where $\zeta_i \sim \mathrm{Bernoulli}(\alpha)$ are independent indicators. The eigenvalues of $Y$ coincide with the squared singular values of $V^\top S$; in particular, $\lambda_{\min}(Y) = \sigma_d(V^\top S)^2$ (with the convention $\sigma_d(V^\top S)=0$ if $n<d$). 

We now apply the matrix Chernoff lower-tail inequality to the independent summands \(X_i := \zeta_i u_i u_i^\top\). To justify the application, we first verify that the assumptions of \cite[Corollary~5.2]{tropp2012} are satisfied.

Since \(\zeta_i \ge 0\) and \(u_i u_i^\top\) is a rank-one positive semidefinite matrix, each summand is self-adjoint and satisfies
\[
X_i = \zeta_i u_i u_i^\top \succeq \mathbf{0},
\]
where \(\mathbf{0}\) denotes the \(d\times d\) zero matrix. Moreover, because \(\zeta_i \le 1\), we have the uniform eigenvalue bound
\[
\lambda_{\max}(X_i) = \|X_i\|_2 = \zeta_i \|u_i\|_2^2 \le \|u_i\|_F^2 \le \max_{1\le j\le p}\|u_j\|_F^2 = \|V\|_{2,\infty}^2 = \frac{\mu d}{p} =: R,
\]
which holds for every realization.

Furthermore, the expectation of the sum \(Y := \sum_{i=1}^p X_i\) is
\[
\mathbb{E}[Y] = \sum_{i=1}^p \mathbb{E}[X_i] = \sum_{i=1}^p \alpha u_i u_i^\top = \alpha I_d,
\]
so \(\mu_{\min} := \lambda_{\min}(\mathbb{E}[Y]) = \alpha\).

Since all conditions of the matrix Chernoff inequality are met, we invoke the simplified lower-tail bound from \cite[Remark~5.3]{tropp2012}, which states that for any \(t \in [0,1]\),
\[
\mathbb{P}\left(\lambda_{\min}(Y) \le t\, \mu_{\min}\right)
\le d \exp\left(-\frac{(1-t)^2 \mu_{\min}}{2R}\right).
\]

Setting \(t = 1/2\) and substituting \(\mu_{\min} = \alpha\) and \(R = \mu d/p\), we obtain
\[
\mathbb{P}\left(\lambda_{\min}(Y) \le \frac{\alpha}{2}\right)
\le d \exp\left(-\frac{(1 - 1/2)^2 \alpha}{2(\mu d / p)}\right)
= d \exp\left(-\frac{\alpha p}{8\mu d}\right).
\]

Substituting condition \eqref{eq:alpha_subspace} gives
\begin{align*}
\frac{\alpha p}{8\mu d} \ge \frac{16(\mu d \log p / p)\cdot p}{8\mu d} = 2\log p.
\end{align*}

Therefore, we obtain
\begin{align*}
\mathbb{P}\left(\lambda_{\min}(Y) \le \frac{\alpha}{2}\right)
\le d \exp(-2\log p) = d p^{-2}.
\end{align*}

Hence, with probability at least $1 - d p^{-2}$, we have
\begin{align*}
\lambda_{\min}(Y) \ge \frac{\alpha}{2}. 
\end{align*}

On this event, \(\lambda_{\min}(Y) > 0\), so \(Y\) is positive definite and therefore has full rank \(d\). Since
\(Y = V^\top S S^\top V = (V^\top S)(V^\top S)^\top\), we have \(\operatorname{rank}(Y) = \operatorname{rank}(V^\top S)\). Hence \(\operatorname{rank}(V^\top S) = d\). Because \(V^\top S \in \mathbb{R}^{d\times n}\), its rank cannot exceed its number of columns. Thus, we must have \(n \ge d\). Moreover, from the identity  $\lambda_{\min}(Y) = \sigma_d(V^\top S)^2$, we have
\begin{align*}
\sigma_d(V^\top S) \ge \sqrt{\frac{\alpha}{2}}.
\end{align*}
\end{proof}
\subsection{Proof of Lemma~\ref{lem:wedin}}
\begin{proof}
Since $\operatorname{rank}(B)=d$, we have $\sigma_{d+1}(B)=0$. Let 
\[
\Delta := \frac12 \sigma_d(B).
\]

By Weyl's inequality for singular values, we have
\[
\sigma_d(H) \ge \sigma_d(B) - \|W\|_2 \ge \sigma_d(B) - \frac12\sigma_d(B) = \Delta > 0.
\]

Thus, applying Wedin's generalized $\sin\theta$ theorem \cite[Section 3]{wedin1972} with the local parameters $\omega := 0$ and $\eta := \Delta$, we have
\[
\sigma_d(H) \ge \omega + \eta,\qquad \sigma_{d+1}(B)=0 \le \omega,
\]
which exactly matches the required gap condition. Hence Wedin's  generalized $\sin\theta$ theorem is applicable.

Let \(\hat V_H \hat{\Lambda}_H \hat U_H^\top\) be the compact SVD of \(H\) for its \(d\) leading singular values, where \(\hat{\Lambda}_H \in \mathbb{R}^{d\times d}\) is the diagonal matrix of the \(d\) nonzero singular values of \(H\) (hence positive definite), and \(\hat V_H \in \mathbb{R}^{p\times d}\), \(\hat U_H \in \mathbb{R}^{n\times d}\) are the matrices of left and right singular vectors, respectively, with orthonormal columns. Define the residuals
\[
R_{11} := B\hat U_H - \hat V_H \hat{\Lambda}_H,\qquad
R_{21} := B^\top \hat V_H - \hat U_H \hat{\Lambda}_H.
\]

Since \(B = H - W\), using the SVD relations \(H\hat U_H = \hat V_H \hat{\Lambda}_H\) and \(H^\top \hat V_H = \hat U_H \hat{\Lambda}_H\), we get
\[
R_{11} = -W\hat U_H,\qquad R_{21} = -W^\top \hat V_H.
\]

Because \(\hat U_H\) and \(\hat V_H\) have orthonormal columns, we have
\[
\|R_{11}\|_F \le \min\{\sqrt d\,\|W\|_2,\|W\|_F\},\qquad
\|R_{21}\|_F \le \min\{\sqrt d\,\|W\|_2,\|W\|_F\},
\]
which gives
\[
\max\{\|R_{11}\|_F,\|R_{21}\|_F\}
\le \min\{\sqrt d\,\|W\|_2,\|W\|_F\}.
\]

Here and throughout, let $\Theta(\hat V_H, V_B)$ denote the diagonal matrix of principal angles between the column spaces of $\hat V_H$ and $V_B$; by the standard identity for subspaces, $\|\sin\Theta(\hat V_H, V_B)\|_F = \|(I - V_B V_B^\top)\hat V_H\|_F$. With the residual bound established above, we now apply Wedin's generalized $\sin\theta$ theorem \cite[Section 3]{wedin1972}, which holds for any unitarily invariant norm. Taking the Frobenius norm in Wedin's theorem gives
\[
\|\sin\Theta(\hat V_H, V_B)\|_F
\le \frac{\max\{\|R_{11}\|_F,\|R_{21}\|_F\}}{\Delta}.
\]

Substituting the upper bound on the residuals (and recalling $\Delta = \frac12\sigma_d(B)$) yields
\[
\|\sin\Theta(\hat V_H, V_B)\|_F
\le \frac{\min\{\sqrt d\,\|W\|_2,\|W\|_F\}}{\frac12\sigma_d(B)}
= \frac{2\min\{\sqrt d\,\|W\|_2,\|W\|_F\}}{\sigma_d(B)}.
\]

Finally, since both $V_B$ and $\hat V_H$ have orthonormal columns spanning $d$-dimensional subspaces, the cosine-sine (CS) decomposition (see \cite[Equation (A5)]{yu2015}) guarantees the existence of an orthogonal matrix $\hat O \in \mathbb{R}^{d\times d}$ such that
\[
\|\hat V_H \hat O - V_B\|_F \le \sqrt{2}\,\|\sin\Theta(\hat V_H, V_B)\|_F.
\]

Combining the inequalities gives
\[
\|\hat V_H \hat O - V_B\|_F
\le \sqrt{2} \cdot \frac{2\min\{\sqrt d\,\|W\|_2,\|W\|_F\}}{\sigma_d(B)}
= \frac{2\sqrt{2}\,\min\{\sqrt d\,\|W\|_2,\|W\|_F\}}{\sigma_d(B)},
\]
which is exactly the claimed bound.
\end{proof}
\subsection{Proof of Theorem~\ref{thm:main}}
\begin{proof}
We treat the cases \(\alpha=1\) and \(0<\alpha<1\) separately.

\medskip
\noindent\textbf{Case 1: \(\alpha=1\).}  
Then \(S=I_p\) deterministically, so \(H=\hat{\Sigma}\). Applying the classical Davis-Kahan theorem for symmetric matrices (see \cite[Theorem 2, Equation (3)]{yu2015}) directly to the pair \((\Sigma, \hat{\Sigma})\) yields the desired bound with constant \(2\sqrt{2}\). Crucially, this classical result requires no upper bound on \(\|E\|_2\) other than the inherent assumption \(\delta>0\), and no probabilistic argument is needed in this degenerate case.

\medskip
\noindent\textbf{Case 2: \(0<\alpha<1\).}  
We prove the bound with probability at least \(1-2p^{-1}\). All constants below are absolute.

\medskip
\noindent\textbf{Step 0: Sample size control.}  
Let \(n = \sum_{i=1}^p \zeta_i\), where \(\zeta_i \sim \mathrm{Bernoulli}(\alpha)\) independently. Note that \(\mu \ge 1\) always holds: indeed, since \(V^\top V = I_d\), we have
\[
\sum_{i=1}^p \|V_{i,:}\|_F^2 = \|V\|_F^2 = d,
\]
so the maximum row squared norm is at least the average, \(\|V\|_{2,\infty}^2 \ge d/p\), which gives \(\mu = (p/d)\|V\|_{2,\infty}^2 \ge 1\). Hence the condition \(\alpha \ge 16\mu d\log p/p\) implies \(\alpha \ge 16d\log p/p\). Applying Lemma~\ref{lem:n_lower_bound} with this condition yields
\[
\mathbb{P}(n \ge d) \ge 1 - p^{-2d}.
\]

Let \(\mathcal E_0 := \{n\ge d\}\); then \(\mathbb{P}(\mathcal E_0) \ge 1 - p^{-2d}\). On this event, the leading \(d\) left singular vectors of \(H\) are well-defined.

\medskip
\noindent\textbf{Step 1: Subspace preservation.}  
Lemma~\ref{lem:subspace} applies and yields, with probability at least \(1 - d p^{-2}\), 
\begin{align*}
\sigma_d(V^\top S) \ge \sqrt{\frac{\alpha}{2}}.
\end{align*}

Let \(\mathcal E_1\) be this event; then \(\mathbb{P}(\mathcal E_1) \ge 1 - d p^{-2}\).

\medskip
\noindent\textbf{Step 2: Lower bound for \(B = \Sigma S\).}  
On the event \(\mathcal E_1\), Lemma~\ref{lem:subspace} ensures that the \(d\times n\) matrix \(V^\top S\) has full row rank \(d\). Since both \(V\) (orthonormal columns) and \(\Lambda\) (invertible diagonal matrix) have full column rank \(d\), the product
\[
B = \Sigma S = V\Lambda (V^\top S)
\]
also has rank \(d\). 

To obtain a quantitative lower bound on its smallest nonzero singular value, consider the \(n\times n\) Gram matrix
\[
B^\top B = S^\top \Sigma^2 S = (V^\top S)^\top \Lambda^2 (V^\top S).
\]

Since \(\delta = \min_i |\lambda_i|\), we have \(\Lambda^2 \succeq \delta^2 I_d\) in the Loewner order. Consequently, we have
\[
B^\top B \succeq \delta^2 (V^\top S)^\top (V^\top S).
\]

By Weyl's monotonicity theorem for eigenvalues of Hermitian matrices, taking the \(d\)-th largest eigenvalue on both sides yields
\[
\lambda_d(B^\top B) \ge \delta^2 \lambda_d((V^\top S)^\top (V^\top S)).
\]

Recognizing that \(\lambda_d(B^\top B) = \sigma_d(B)^2\) and \(\lambda_d((V^\top S)^\top (V^\top S)) = \sigma_d(V^\top S)^2\), and applying Lemma~\ref{lem:subspace} on the event \(\mathcal E_1\), we obtain
\begin{align}
\sigma_d(B) \ge \delta \,\sigma_d(V^\top S) \ge \delta \sqrt{\frac{\alpha}{2}}. \label{eq:sigmaB}
\end{align}

\medskip
\noindent\textbf{Step 3: Application of Lemma~\ref{lem:wedin}.}  
Let \(W = E S\). Since \(\|S\|_2=1\), we have \(\|W\|_2 \le \|E\|_2\) and \(\|W\|_F \le \|E\|_F\). The condition \(\alpha \ge 8(\|E\|_2/\delta)^2\) gives \(\|E\|_2 \le \delta\sqrt{\alpha/8}\). From Equation \eqref{eq:sigmaB}, we get
\[
\frac12\sigma_d(B) \ge \frac{\delta}{2}\sqrt{\frac{\alpha}{2}}\ge \|E\|_2,
\]
so \(\|W\|_2 \le \frac12\sigma_d(B)\). Thus Lemma~\ref{lem:wedin} applies and yields an orthogonal matrix \(\hat O_1 \in \mathbb{R}^{d\times d}\) such that
\begin{align*}
\|\hat V_H \hat O_1 - V_B\|_F
\le \frac{2\sqrt{2}\,\min\{\sqrt{d}\,\|W\|_2,\|W\|_F\}}{\sigma_d(B)}. 
\end{align*}

Substituting Equation \eqref{eq:sigmaB} and the bounds on \(\|W\|_2,\|W\|_F\) gives
\begin{align}
\|\hat V_H \hat O_1 - V_B\|_F
\le \frac{2\sqrt{2}\,\min\{\sqrt{d}\,\|E\|_2,\|E\|_F\}}{\delta\sqrt{\alpha/2}}
= \frac{4\min\{\sqrt{d}\,\|E\|_2,\|E\|_F\}}{\delta\sqrt{\alpha}}. \label{eq:step3_final}
\end{align}

\medskip
\noindent\textbf{Step 4: Alignment of \(V_B\) with \(V\).}  
On the event \(\mathcal E_1\), \(V^\top S\) has rank \(d\). Write
\[
B = \Sigma S = V\Lambda (V^\top S) = V M,
\]
where \(M = \Lambda (V^\top S) \in \mathbb{R}^{d\times n}\). Since \(\Lambda\) is invertible and \(V^\top S\) has rank \(d\), the matrix \(M\) also has rank \(d\) (i.e., full row rank). Therefore, its \(d\) rows form a basis of \(\mathbb{R}^d\). Consequently, for every vector \(a\in\mathbb{R}^d\), there exists a coefficient vector \(x\in\mathbb{R}^n\) such that
\begin{align}\label{eq:step4}
M x = a.
\end{align}

Now fix any column \(v_i = V e_i\) of \(V\), where \(e_i\in\mathbb{R}^d\) is the \(i\)-th standard basis vector. By Equation (\ref{eq:step4}) with \(a=e_i\), there exists \(x_i\in\mathbb{R}^n\) such that \(M x_i = e_i\). Then, we have
\[
B x_i = V M x_i = V e_i = v_i.
\]

Thus every column of \(V\) is a linear combination of the columns of \(B\), i.e., \(\operatorname{span}(V)\subseteq \operatorname{span}(B)\). The reverse inclusion follows immediately from \(B = V M\), since every column of \(B\) is a linear combination of the columns of \(V\). Hence, we obtain \(\operatorname{span}(B) = \operatorname{span}(V)\).

Since the leading \(d\) left singular vectors \(V_B\) of \(B\) form an orthonormal basis of \(\operatorname{span}(B)\), and \(V\) is an orthonormal basis of the same subspace, there exists an orthogonal matrix \(O_B\in\mathbb{R}^{d\times d}\) such that \(V = V_B O_B\). Setting \(\hat O = \hat O_1 O_B^\top\) and applying Equation \eqref{eq:step3_final} yields the desired bound.

Finally, by the union bound, we obtain
\[
\mathbb{P}(\mathcal E_0 \cap \mathcal E_1)
\ge 1 - p^{-2d} - d p^{-2},
\]
which completes the proof.
\end{proof}

\subsection{Proof of Theorem \ref{thm:sbm}}
\begin{proof}
We verify the conditions of Theorem~\ref{thm:main} under the SBM setting. Since \(d=O(1)\), all constants depending on \(d\) will be absorbed into the \(O(\cdot)\) notation.

First, for the eigenvector matrix \(V\in\mathbb{R}^{p\times d}\) of \(\Sigma\), since \(\Sigma = \rho Z B_0 Z^\top\), we can write \(V = Z (Z^\top Z)^{-1/2} X\) for some orthogonal matrix \(X\in\mathbb{R}^{d\times d}\). For node \(i\) in community \(k\), we have
\begin{align*}
\|V_{i,:}\|_F = \frac{1}{\sqrt{p_k}} = O\!\left(\frac{1}{\sqrt{p}}\right),
\end{align*}
so \(\|V\|_{2,\infty}^2 = O(1/p)\) and hence
\begin{align}
\mu = \frac{p}{d}\|V\|_{2,\infty}^2 = O(1). \label{eq:mu_sbm}
\end{align}

Second, the smallest absolute eigenvalue of \(\Sigma\) satisfies
\begin{align}
\delta = \rho \cdot \lambda_{\min}(|Z B_0 Z^\top|) \asymp \rho \cdot p \cdot \lambda_{\min}(|B_0|) \asymp \rho p, \label{eq:delta_sbm}
\end{align}
where the equivalence follows from balanced communities and the assumption on \(B_0\).

Third, by Lemma 2.2 of \cite{jin2015}, the spectral norm of the noise satisfies
\begin{align}
\|E\|_2 = \|\widehat{\Sigma}-\Sigma\|_2 = O(\sqrt{\rho p \log p}) \label{eq:eps_sbm}
\end{align}
with high probability.

Substituting Equations \eqref{eq:mu_sbm}, \eqref{eq:delta_sbm}, and \eqref{eq:eps_sbm} into the condition of Theorem~\ref{thm:main}, we obtain
\begin{align*}
\alpha \ge \max\left\{16 \frac{\mu d\log p}{p},\; 8 \frac{\|E\|_2^2}{\delta^2}\right\}
= \max\left\{O\!\left(\frac{\log p}{p}\right),\; O\!\left(\frac{\log p}{\rho p}\right)\right\},
\end{align*}
where the first term uses \(\mu d = O(1)\). Thus the assumed condition \(\alpha \ge C\log p/(\rho p)\) is sufficient for Theorem~\ref{thm:main} to apply.

Applying Theorem~\ref{thm:main} with \(\mu=O(1)\), \(\delta \asymp \rho p\), and \(\|E\|_2 = O(\sqrt{\rho p \log p})\), we get
\begin{align}
\|\hat V_H \hat O - V\|_F = O\!\left(\frac{\sqrt{\log p}}{\sqrt{\rho\, p\, \alpha}}\right) \label{eq:subspace_error_sbm}
\end{align}
for some orthogonal matrix \(\hat O\in\mathbb{R}^{d\times d}\).

Now we apply Lemma 5.3 of \cite{lei2015}. In the balanced setting with \(d=O(1)\), the minimum row separation of \(V\) is \(O(1/\sqrt{p})\), and their Lemma 5.3 yields a misclustering rate bounded by \(O(\|\hat V_H\hat O - V\|_F^2)\) (the factor \(1/d\) is absorbed since \(d=O(1)\)). Therefore, we obtain
\begin{align*}
\frac{1}{p}\min_{\pi}\sum_{i=1}^p \mathbf{1}\{\hat{g}_i \neq \pi(g_i)\}
= O(\|\hat V_H\hat O - V\|_F^2)
= O\!\left(\frac{\log p}{\rho\, p\, \alpha}\right),
\end{align*}
where the last equality uses Equation \eqref{eq:subspace_error_sbm}. 
\end{proof}

\bibliographystyle{elsarticle-num}
\bibliography{refsubDK}

@article{davis1970,
  title={The rotation of eigenvectors by a perturbation. III},
  author={Davis, Chandler and Kahan, William Morton},
  journal={SIAM Journal on Numerical Analysis},
  volume={7},
  number={1},
  pages={1--46},
  year={1970},
  publisher={SIAM}
}

@article{rohe2011,
	title="Spectral clustering and the high-dimensional stochastic blockmodel",
	author="Karl {Rohe} and Sourav {Chatterjee} and Bin {Yu}",
	journal="Annals of Statistics",
	volume="39",
	number="4",
	pages="1878--1915",
	year="2011"
}

@article{tropp2012,
	title="User-Friendly Tail Bounds for Sums of Random Matrices",
	author="Joel A. {Tropp}",
	journal="Foundations of Computational Mathematics",
	volume="12",
	number="4",
	pages="389--434",
	year="2012"
}

@article{wedin1972,
  title={Perturbation bounds in connection with singular value decomposition},
  author={Wedin, Per-{\AA}ke},
  journal={BIT Numerical Mathematics},
  volume={12},
  number={1},
  pages={99--111},
  year={1972},
  publisher={Springer}
}

@article{yu2015,
	title="A useful variant of the Davis--Kahan theorem for statisticians",
	author="Yi {Yu} and Tengyao {Wang} and Richard J. {Samworth}",
	journal="Biometrika",
	volume="102",
	number="2",
	pages="315--323",
	year="2015"
}

@article{lei2015,
	title="Consistency of spectral clustering in stochastic block models",
	author="Jing {Lei} and Alessandro {Rinaldo}",
	journal="Annals of Statistics",
	volume="43",
	number="1",
	pages="215--237",
	year="2015"
}

@article{joseph2016,
	title="Impact of regularization on spectral clustering",
	author="Antony {Joseph} and Bin {Yu}",
	journal="Annals of Statistics",
	volume="44",
	number="4",
	pages="1765--1791",
	year="2016"
}

@article{jin2015,
	title={{Fast community detection by SCORE}},
	author="Jiashun {Jin}",
	journal="Annals of Statistics",
	volume="43",
	number="1",
	pages="57--89",
	year="2015"
}

@inproceedings{qin2013,
	title={{Regularized spectral clustering under the degree-corrected stochastic blockmodel}},
	author="Tai {Qin} and Karl {Rohe}",
	booktitle="Advances in Neural Information Processing Systems 26",
	pages="3120--3128",
	year="2013"
}

@article{guo2024efficacy,
  title={On the efficacy of higher-order spectral clustering under weighted stochastic block models},
  author={Guo, Xiao and Zhang, Hai and Chang, Xiangyu},
  journal={Computational Statistics $\&$ Data Analysis},
  volume={190},
  pages={107872},
  year={2024},
  publisher={Elsevier}
}

@article{lei2023,
  title={Bias-adjusted spectral clustering in multi-layer stochastic block models},
  author={Lei, Jing and Lin, Kevin Z},
  journal={Journal of the American Statistical Association},
  volume={118},
  number={544},
  pages={2433--2445},
  year={2023},
  publisher={Taylor \& Francis}
}

@article{binkiewicz2017,
  title={Covariate-assisted spectral clustering},
  author={Binkiewicz, Norbert and Vogelstein, Joshua T and Rohe, Karl},
  journal={Biometrika},
  volume={104},
  number={2},
  pages={361--377},
  year={2017},
  publisher={Oxford University Press}
}

@article{chen2021spectral,
	title="Spectral Methods for Data Science: A Statistical Perspective",
	author="Yuxin {Chen} and Yuejie {Chi} and Jianqing {Fan} and Cong {Ma}",
	journal="Foundations and Trends{\textregistered} in Machine Learning",
	volume="14",
	number="5",
	pages="566--806",
	year="2021"
}

@article{deng2024,
  title={Subsampling spectral clustering for stochastic block models in large-scale networks},
  author={Deng, Jiayi and Huang, Danyang and Ding, Yi and Zhu, Yingqiu and Jing, Bingyi and Zhang, Bo},
  journal={Computational Statistics $\&$ Data Analysis},
  volume={189},
  pages={107835},
  year={2024},
  publisher={Elsevier}
}

@article{ding2024,
  title={Stock co-jump networks},
  author={Ding, Yi and Li, Yingying and Liu, Guoli and Zheng, Xinghua},
  journal={Journal of Econometrics},
  volume={239},
  number={2},
  pages={105420},
  year={2024},
  publisher={Elsevier}
}

@inproceedings{drineas2006,
  title={Sampling algorithms for l 2 regression and applications},
  author={Drineas, Petros and Mahoney, Michael W and Muthukrishnan, Shan},
  booktitle={Proceedings of the seventeenth annual ACM-SIAM symposium on Discrete algorithm},
  pages={1127--1136},
  year={2006}
}

@article{drineas2011,
  title={Faster least squares approximation},
  author={Drineas, Petros and Mahoney, Michael W and Muthukrishnan, Shan and Sarl{\'o}s, Tam{\'a}s},
  journal={Numerische Mathematik},
  volume={117},
  number={2},
  pages={219--249},
  year={2011},
  publisher={Springer}
}

@article{drineas2012,
  title={Fast approximation of matrix coherence and statistical leverage},
  author={Drineas, Petros and Magdon-Ismail, Malik and Mahoney, Michael W and Woodruff, David P},
  journal={Journal of Machine Learning Research},
  volume={13},
  number={1},
  pages={3475--3506},
  year={2012},
  publisher={JMLR. org}
}

@article{drineas2016,
  title={RandNLA: randomized numerical linear algebra},
  author={Drineas, Petros and Mahoney, Michael W},
  journal={Communications of the ACM},
  volume={59},
  number={6},
  pages={80--90},
  year={2016},
  publisher={ACM New York, NY, USA}
}

@article{halko2011,
  title={Finding structure with randomness: Probabilistic algorithms for constructing approximate matrix decompositions},
  author={Halko, Nathan and Martinsson, Per-Gunnar and Tropp, Joel A},
  journal={SIAM Review},
  volume={53},
  number={2},
  pages={217--288},
  year={2011},
  publisher={SIAM}
}

@article{mahoney2011,
  title={Randomized algorithms for matrices and data},
  author={Mahoney, Michael W},
  journal={Foundations and Trends{\textregistered} in Machine Learning},
  volume={3},
  number={2},
  pages={123--224},
  year={2011},
  publisher={Emerald Publishing Limited}
}

@article{martinsson2016,
  title={Randomized methods for matrix computations},
  author={Martinsson, Per-Gunnar},
  journal={arXiv preprint arXiv:1607.01649},
  year={2016}
}

@article{qing2025a,
  title={Community detection by spectral methods in multi-layer networks},
  author={Qing, Huan},
  journal={Applied Soft Computing},
  volume={171},
  pages={112769},
  year={2025},
  publisher={Elsevier}
}

@article{qing2025b,
  title={Community detection in multi-layer networks by regularized debiased spectral clustering},
  author={Qing, Huan},
  journal={Engineering Applications of Artificial Intelligence},
  volume={152},
  pages={110627},
  year={2025},
  publisher={Elsevier}
}

@article{su2024,
  title={Spectral co-clustering in multi-layer directed networks},
  author={Su, Wenqing and Guo, Xiao and Chang, Xiangyu and Yang, Ying},
  journal={Computational Statistics $\&$ Data Analysis},
  volume={198},
  pages={107987},
  year={2024},
  publisher={Elsevier}
}

@article{vonluxburg2007,
  title={A tutorial on spectral clustering},
  author={Von Luxburg, Ulrike},
  journal={Statistics and Computing},
  volume={17},
  number={4},
  pages={395--416},
  year={2007},
  publisher={Springer}
}

@article{witten2015,
  title={Randomized algorithms for low-rank matrix factorizations: sharp performance bounds},
  author={Witten, Rafi and Candes, Emmanuel},
  journal={Algorithmica},
  volume={72},
  number={1},
  pages={264--281},
  year={2015},
  publisher={Springer}
}

@article{zhang2022randomized,
  title={Randomized spectral clustering in large-scale stochastic block models},
  author={Zhang, Hai and Guo, Xiao and Chang, Xiangyu},
  journal={Journal of Computational and Graphical Statistics},
  volume={31},
  number={3},
  pages={887--906},
  year={2022},
  publisher={Taylor \& Francis}
}

@article{qi2025detection,
  title={Detection of model-based planted pseudo-cliques in random dot product graphs by the adjacency spectral embedding and the graph encoder embedding},
  author={Qi, Tong and Lyzinski, Vince},
  journal={IEEE Transactions on Pattern Analysis and Machine Intelligence},
  year={2025},
  publisher={IEEE}
}

@article{weylandt2026multivariate,
  title={Multivariate analysis for multiple network data via semi-symmetric tensor pca},
  author={Weylandt, Michael and Michailidis, George},
  journal={Journal of the American Statistical Association},
  pages={1--14},
  year={2026},
  publisher={Taylor \& Francis}
}

@article{srivastava2023robust,
  title={A robust spectral clustering algorithm for sub-Gaussian mixture models with outliers},
  author={Srivastava, Prateek R and Sarkar, Purnamrita and Hanasusanto, Grani A},
  journal={Operations Research},
  volume={71},
  number={1},
  pages={224--244},
  year={2023},
  publisher={INFORMS}
}

@article{xu2023covariate,
  title={Covariate-assisted community detection in multi-layer networks},
  author={Xu, Shirong and Zhen, Yaoming and Wang, Junhui},
  journal={Journal of Business $\&$ Economic Statistics},
  volume={41},
  number={3},
  pages={915--926},
  year={2023},
  publisher={Taylor \& Francis}
}

@article{cucuringu2021regularized,
  title={Regularized spectral methods for clustering signed networks},
  author={Cucuringu, Mihai and Singh, Apoorv Vikram and Sulem, D{\'e}borah and Tyagi, Hemant},
  journal={Journal of Machine Learning Research},
  volume={22},
  number={264},
  pages={1--79},
  year={2021}
}

@article{holland1983stochastic,
	title="Stochastic blockmodels: First steps",
	author="Paul W. {Holland} and Kathryn Blackmond {Laskey} and Samuel {Leinhardt}",
	journal="Social Networks",
	volume="5",
	number="2",
	pages="109--137",
	year="1983"
}

@article{karrer2011stochastic,
	title="Stochastic blockmodels and community structure in networks",
	author="Brian {Karrer} and M. E. J. {Newman}",
	journal="Physical Review E",
	volume="83",
	number="1",
	pages="16107",
	year="2011"
}

@article{bhadra2026unified,
  title={A unified framework for community detection and model selection in blockmodels},
  author={Bhadra, Subhankar and Tang, Minh and Sengupta, Srijan},
  journal={Journal of Computational and Graphical Statistics},
  volume={35},
  number={3},
  pages={1036--1049},
  year={2026},
  publisher={Taylor \& Francis}
}

@article{wu2023distributed,
  title={A distributed community detection algorithm for large scale networks under stochastic block models},
  author={Wu, Shihao and Li, Zhe and Zhu, Xuening},
  journal={Computational Statistics $\&$ Data Analysis},
  volume={187},
  pages={107794},
  year={2023},
  publisher={Elsevier}
}

@article{paul2020spectral,
author = {Subhadeep Paul and Yuguo Chen},
title = {{Spectral and matrix factorization methods for consistent community detection in multi-layer networks}},
volume = {48},
journal = {Annals of Statistics},
number = {1},
publisher = {Institute of Mathematical Statistics},
pages = {230 -- 250},
year = {2020}
}

@article{rubindelanchy2022statistical,
  title={A statistical interpretation of spectral embedding: the generalised random dot product graph},
  author={Rubin-Delanchy, Patrick and Cape, Joshua and Tang, Minh and Priebe, Carey E},
  journal={Journal of the Royal Statistical Society Series B: Statistical Methodology},
  volume={84},
  number={4},
  pages={1446--1473},
  year={2022},
  publisher={Oxford University Press}
}

@article{lin2026dynamic,
  title={Dynamic clustering for heterophilic stochastic block models with time-varying node memberships},
  author={Lin, Kevin Z and Lei, Jing},
  journal={Biometrika},
  volume={113},
  number={2},
  pages={asag018},
  year={2026},
  publisher={Oxford University Press}
}

@article{su2025randomized,
  title={Randomized Spectral Clustering for Large-Scale Multi-Layer Networks},
  author={Su, Wenqing and Guo, Xiao and Chang, Xiangyu and Yang, Ying},
  journal={Statistics and Computing},
  volume={35},
  number={6},
  pages={190},
  year={2025},
  publisher={Springer}
}
\end{document}